\documentclass{article}
\usepackage{iclr2027_conference,times}
\usepackage{xcolor}
\usepackage{amsmath,amssymb,mathtools,amsthm}
\usepackage{graphicx}
\usepackage{booktabs}
\usepackage{longtable}
\usepackage{pdflscape}
\usepackage{array}
\usepackage{tabularx}
\usepackage{enumitem}
\usepackage{microtype}
\usepackage{xspace}
\usepackage{url}
\usepackage[colorlinks=true,allcolors=blue,hypertexnames=false]{hyperref}
\usepackage[nameinlink,noabbrev]{cleveref}
\usepackage{autoresearch_algorithm}
\usepackage{float}
\usepackage{placeins}

\newtheorem{proposition}{Proposition}
\newcommand{\Build}{\operatorname{Build}}
\newcommand{\Run}{\operatorname{Run}}

\setlist[itemize]{leftmargin=*,topsep=2pt,itemsep=1pt,parsep=0pt}
\setlist[enumerate]{leftmargin=*,topsep=2pt,itemsep=1pt,parsep=0pt}
\renewcommand{\arraystretch}{1.08}

\title{DAGLoop: Aligning Method Design\\with Experimental Feedback\\in Automated AI Research}
\iclrfinalcopy
\author{%
  \normalfont\normalsize
  Dengyuxing Wang\textsuperscript{1}\quad
  Yonggang Zhang\textsuperscript{1,3}\quad
  Qianshu Cai\textsuperscript{1,4}\\[4pt]
  Huajiang Zheng\textsuperscript{1,3}\quad
  Jun Song\textsuperscript{2}\quad
  Yangqiu Song\textsuperscript{1}\quad
  Yike Guo\textsuperscript{1,3}\\[7pt]
  {\small\textsuperscript{1}The Hong Kong University of Science and Technology}\\
  {\small\textsuperscript{2}Hong Kong Baptist University}\\
  {\small\textsuperscript{3}Hong Kong Generative AI Research \& Development Center}\\
  {\small\textsuperscript{4}University of Science and Technology of China}\\[3pt]
}

\begin{document}
\maketitle
\pagestyle{plain}

\begin{abstract}
Automated AI research agents aim to improve machine learning methods by iteratively proposing, implementing, and evaluating candidate methods. Existing work uses the outcome of each evaluation to determine how the next candidate should be revised for performance improvement. However, a candidate often combines multiple design decisions, and its overall outcome alone does not reveal which decisions produced their expected effects, whether their required conditions were satisfied, or which decisions should be retained or revised. Inspired by ablation studies in the machine learning community, we propose DAGLoop to test the expected effect of a design decision by comparing implementations with and without that decision while preserving the other decisions it requires. Specifically, DAGLoop translates each design decision’s rationale into a testable expectation and pairs it with a targeted experiment, forming an experimental specification for assessing the expected effect under its required conditions. Using these specifications, we formulate ablation planning as the construction of a directed acyclic graph (DAG) from the candidate method’s design: nodes associate decisions with tests, and edges encode dependencies. The DAG specifies which decisions each experiment must retain and supports joint tests when an expected effect requires several decisions to act together. Outcomes are recorded with the tested decisions, expectations, and conditions to form structured research memory. This memory guides the agent in retaining experimentally supported decisions, revising unsupported ones, and planning further experiments when the evidence is unresolved. Experiments on dozens of research tasks demonstrate substantial performance improvements achieved by DAGLoop across diverse executable research settings. Code and evaluation scripts are available at \url{https://github.com/deng34567/DAGLoop}.
\end{abstract}

\section{Introduction}

Improving a machine learning method requires repeated decisions about what to change, what to test, and how to interpret the result. Recent research agents automate this process by proposing candidate methods, implementing them, running experiments, and using the outcomes to guide subsequent revisions. MLAgentBench studies agents that conduct machine learning experiments, while The AI Scientist extends automation to longer research workflows and agentic tree search \citep{huang2023mlagentbench,lu2026endtoend,yamada2025aiscientistv2}. AlphaEvolve uses executable evaluators to guide program and algorithm discovery \citep{novikov2025alphaevolve}. We study the executable method-improvement setting: given a research task, a baseline, an experimental environment, and a finite budget, an agent seeks to improve performance under a fixed task evaluator.

Experimental feedback determines which candidates the agent explores next. Optimization by PROmpting (OPRO) generates candidates from scored solutions; FunSearch selects evaluated programs for further search \citep{yang2024opro,romeraparedes2024funsearch}. Richer feedback includes MetaReflection's memory of past reflections and Trace's execution traces for parameter updates \citep{gupta2024metareflection,cheng2024trace}. Our focus is the relation between this feedback and the design decisions being revised. A candidate may combine a data transformation, a representation, an objective, and an optimization rule. Its overall performance measures their combined effect, but does not establish which decisions produced their expected effects or whether their required conditions held. The evaluator scores the complete candidate, whereas revision must determine which constituent decisions to retain, revise, or test further.

Ablation studies examine decisions by comparing implementations with and without them; recent benchmarks study whether language models can plan informative ablations \citep{zhao2025abgen,abramovich2025ablationbench}. However, testing each decision independently can remove the conditions needed for its effect. Suppose $z_1$ constructs a representation and $z_2$ introduces an objective that uses it. The representation may provide a necessary intermediate property without improving the task metric alone, while the objective may be meaningless without it. Rejecting $z_1$ for lacking standalone gain prevents a meaningful test of $z_2$; evaluating only their composition leaves each decision's contribution unresolved. Decision-level evidence therefore requires experiments that preserve the prerequisites of the decision being tested (Figure~\ref{fig:feedback_alignment}).

\begin{figure}[t]
  \centering
  \includegraphics[width=0.95\linewidth]{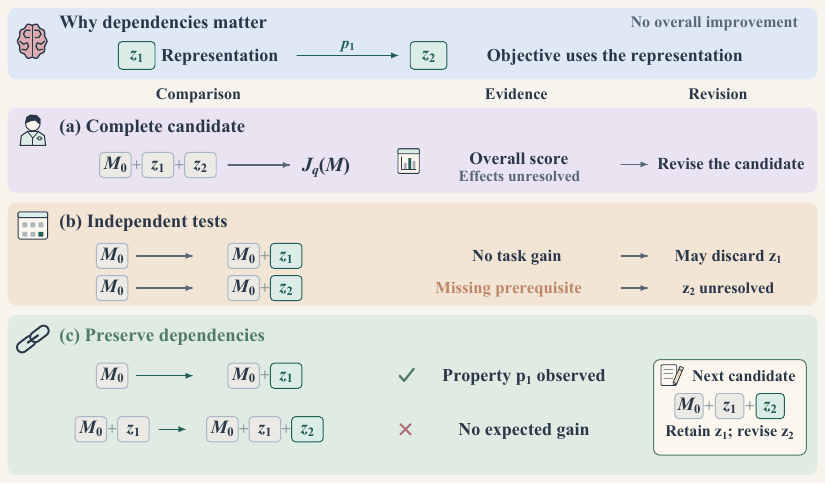}
  \caption{Why dependencies matter. Candidate-level feedback combines decision effects, while independent tests can remove prerequisites. DAGLoop preserves prerequisites in the control. Here $M_0$ is the baseline, $J_q$ the task metric, and $p_1$ a property supplied by $z_1$.}
  \label{fig:feedback_alignment}
\end{figure}

We present \textbf{DAGLoop}, a general-purpose research system that addresses executable method-improvement problems through proposal, implementation, experimentation, and revision. It constructs targeted experiments from a candidate method's design. For each decision, it translates the rationale into a testable expectation: the effect the decision should produce and its required conditions. A targeted experiment specifies what to compare and measure, forming an experimental specification that connects design reasoning to evidence. These specifications define a DAG, whose nodes associate decisions with tests and whose edges encode their dependencies. Each experiment compares a control implementation containing the prerequisites with a treatment implementation that additionally applies the decision: $z_1+z_2$ against $z_1$ in the example. A decision that provides a prerequisite, such as $z_1$, is assessed through the intermediate property it is intended to provide. Effects requiring several design decisions to act together are tested jointly, retaining their external prerequisites and attributing experimental support to the group of decisions tested together.

After checking execution validity, DAGLoop records each outcome alongside the decision, expected effect, required conditions, and saved control implementation. This structured research memory guides the agent to retain supported decisions, revise unsupported ones, and plan further experiments when evidence is unresolved. Failed implementations provide debugging information. Revised decisions are assembled into a subsequent candidate and evaluated with the same official evaluator. Targeted experiments guide revision; complete-candidate evaluation measures improvement.

We evaluate DAGLoop on 36 AutoLab tasks spanning CPU optimization, CUDA kernels, program synthesis, and model development \citep{xu2026autolab}, comparing with Terminus-2 using the same language model. We make the following three contributions:
\begin{itemize}
    \item We identify a mismatch between evaluating complete candidates and revising individual decisions, motivating experiments that retain prerequisites for decision-level evidence.
    \item We propose DAGLoop, a research system that translates design rationales into experimental specifications, constructs a dependency DAG to plan targeted ablations, and uses structured research memory to guide evidence-based revision and further experimentation.
    \item Experiments on 36 AutoLab tasks demonstrate method improvement across diverse executable problems. DAGLoop achieves a mean score of 0.6459, compared with 0.5831 for Terminus-2 using GLM-5.2 with maximum reasoning effort, with higher means on both CPU and GPU tasks. Additional comparisons average two recorded results per task with reasoning disabled for GLM-5.2, DeepSeek-V4-Flash-0731, and DeepSeek-V4.1-Flash.
\end{itemize}

\section{Method}
\label{sec:method}
DAGLoop organizes method proposal, implementation, experimentation, and revision into a research loop. The system turns a candidate's design into experimental specifications, uses their dependencies to construct targeted comparisons, and records the outcomes for revision (Figure~\ref{fig:method_overview}). The assembled candidate is evaluated by the official task evaluator.

\begin{figure}[htbp]
  \centering
  \includegraphics[width=0.90\linewidth]{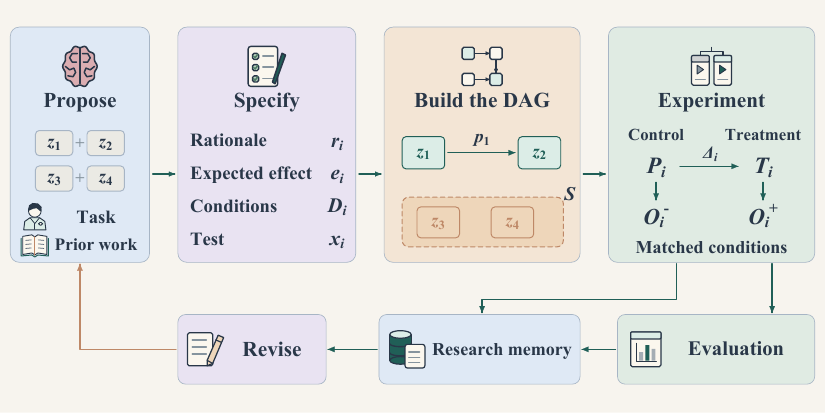}
  \caption{DAGLoop research loop. Each DAG node represents a decision or jointly tested group; $S$ denotes such a group. Control $P_i$ preserves prerequisites; treatment $T_i$ adds the change $\Delta_i$. Their observations $O_i^-$ and $O_i^+$ enter research memory and guide revision. Complete candidates are evaluated separately by the task evaluator, with feedback returned for revision.}
  \label{fig:method_overview}
\end{figure}

\subsection{Task Setup and Experimental Specifications}
\label{sec:preliminary}
\label{sec:problem}
\label{sec:compile}
Starting from the user's research question, DAGLoop clarifies objectives and constraints with the user, retrieves relevant literature, and proposes a method. The user confirms the research direction and method. With an executable baseline, evaluator, environment, and budget established, the system proceeds to implementation, experimentation, and revision.

An executable research task is
\begin{equation}
\mathcal T=(q,M_0,\mathcal E,J_q,B),
\label{eq:task}
\end{equation}
where $q$ describes the task, $M_0$ is an executable baseline, $\mathcal E$ is the experimental environment, and $B$ is the research budget. The official evaluator returns a raw (native) task metric $m_q(M)$, such as accuracy or execution time, and a task score $R_q(M)$. We define $J_q$ so that larger values indicate better performance: $J_q(M)=m_q(M)$ for maximization tasks and $J_q(M)=-m_q(M)$ for minimization tasks. The score $R_q$ follows the official reward rule, including its thresholds and clipping.

Let $\tau$ be the research trajectory, including proposals, implementations, experiments, and revisions, and let $\mathcal C(\tau)$ contain the complete executable candidates it evaluates. We seek
\begin{equation}
\widehat M(\tau)\in
\arg\max_{\substack{M\in\mathcal C(\tau)\cup\{M_0\}\\
                    M\ \text{passes the task checks}}}J_q(M),
\qquad \operatorname{Cost}_{\mathcal E}(\tau)\le B.
\label{eq:taskobjective}
\end{equation}
Task checks enforce correctness and task constraints. Cost includes planning, implementation, targeted tests, repair, repeated measurements, and final evaluation. Improvement means $J_q(\widehat M)>J_q(M_0)$. The baseline remains available when no evaluated candidate improves it.

A candidate method $M$ consists of design decisions $Z=\{z_i\}_{i=1}^n$. The complete-candidate metric $J_q(M)$ evaluates their combined performance, whereas revision requires deciding which $z_i$ to retain or change. To test the effect expected from a particular decision, we associate it with an experimental specification linking its rationale, expected effect, conditions, and test:
\begin{equation}
s_i=(r_i,e_i,D_i,x_i),\qquad
x_i=(\Delta_i,\rho_i,o_i,v_i),
\label{eq:specification}
\end{equation}
where $r_i$ is the rationale, $e_i$ the expected effect, and $D_i$ the required conditions. The experiment $x_i$ specifies the implementation change $\Delta_i$, execution procedure $\rho_i$, quantity to measure $o_i$, and validation rule $v_i$. The specification defines the measurement and criterion before execution; the experiment report supplies the observed value. A decision can provide an intermediate property $p_i$ required by another decision, with a corresponding property check.

Each node in the experimental DAG represents an executable component: one decision or a group of decisions tested together. Each component has an implementation and a specification. Tests assess either improvement in the task metric or an intermediate property required by another decision. We refer to these as \emph{performance tests} and \emph{property checks}, respectively. A checkpoint stores the implementation and artifacts needed to reconstruct a control or candidate.

\paragraph{From a rationale to a test.}
In the representation--objective example, the rationale for $z_1$ identifies a property $p_1$ needed by $z_2$. Its specification makes that property the expected effect and selects an observable that checks whether the representation supplies it. The specification for $z_2$ instead declares a task-metric improvement, with $p_1$ included in its required conditions. These declarations determine both what is measured and which implementation is a meaningful control. A successful property check establishes the context for testing $z_2$; the downstream experiment then assesses its own declared effect. The planner specifies these roles before execution and assesses prerequisite decisions by testing their intended intermediate properties.

The planner uses a language model to propose decisions, dependencies, and tests from the method, task, literature summary, and baseline evidence. A review agent (the critic) checks whether the tests assess the expected effects under the required conditions. The planner revises the proposal using this feedback. The plan is accepted after review finds no outstanding issues and graph validation passes; Section~\ref{sec:plan} describes the graph and the prerequisite checks used to validate it.

When an effect requires several decisions to act together, the planner groups them into one component with a shared specification and test. The plan records the original decisions and their external dependencies. The result is attributed to the group.

\subsection{Ablation Planning}
\label{sec:plan}

Ablation analysis compares algorithm configurations to assess particular changes~\citep{abramovich2025ablationbench}. Fix the baseline, implementations, and evaluation conditions, and let $f(C)$ denote $J_q$ for the implementation containing a feasible set $C$ of decision indices. For $C$ and $C\cup\{i\}$ satisfying the declared dependencies, the conditional difference associated with adding decision $i$ is
\begin{equation}
\delta_{i\mid C}=f(C\cup\{i\})-f(C),\qquad i\notin C.
\label{eq:conditionalgain}
\end{equation}
\looseness=-1
This difference is conditional on the fixed feasible configuration $C$ and compares implementations with and without decision $i$ under the same evaluation conditions. For $z_2$ requiring $z_1$, setting $C=\{1\}$ recovers $\delta_{2\mid1}=f(\{1,2\})-f(\{1\})$. Baseline and complete-candidate values do not determine its sign because they leave $f(\{1\})$ unobserved (Proposition~\ref{prop:endpoint}). Testing $z_1$ supplies this control measurement; the property check assesses whether this implementation provides the declared intermediate property.

Let $V$ index the components. The conditions $D_i$ describe what component $i$ requires; $d_i^0$ identifies the other components whose outputs or properties it directly needs. An edge $(j,i)$ denotes such a dependency. Other conditions, such as the data split and measurement settings, are specified in the experimental protocol. The set $K_i$ contains all ancestors of $i$, including indirect prerequisites:
\begin{equation}
\begin{aligned}
A_0&=\{(j,i):j\in d_i^0\},\qquad G_0=(V,A_0),\\
K_i&=\bigcup_{j\in d_i^0}\bigl(\{j\}\cup K_j\bigr).
\end{aligned}
\label{eq:graphclosure}
\end{equation}
Graph validation checks identifiers, acyclicity, and agreement between $K_i$ and the declared prerequisites (Appendix~\ref{app:details}). The control $P_i$ retains prerequisites; treatment $T_i$ adds the change:
\begin{equation}
P_i=\Build(M_0,K_i),\qquad T_i=P_i\oplus\Delta_i.
\label{eq:parentcomparison}
\end{equation}
$\Build$ combines saved implementations of the required components with the baseline. The operator $\oplus$ applies the tested component's implementation change to this control. Implementation review checks that the resulting code follows the specification. Appendix~\ref{app:details} details this merge.

For $z_1\to z_2$, the test compares $z_1+z_2$ with $z_1$, preserving the representation required by the objective. Testing $z_c$ in $z_a\to z_b\to z_c$ retains both $z_a$ and $z_b$. Thus, the control contains every required component together with its own prerequisites.

For a component $u$ representing a group $S$, the experiment applies the decisions together to a control containing their external prerequisites. The shared specification defines the expected combined effect. The record attributes the result to the tested group and identifies its members. The revised graph must remain acyclic. Appendices~\ref{app:details} and~\ref{app:proofs} give the construction and its structural properties.

\subsection{Executing and Assessing Experiments}
\label{sec:execution}
The procedure $\rho_i$ specifies how to run the experiment; the protocol $\omega_i$ specifies the evaluation conditions and measurement settings shared by control and treatment:
\begin{equation}
O_i^-=\Run(\mathcal E,P_i;\rho_i,\omega_i),\qquad
O_i^+=\Run(\mathcal E,T_i;\rho_i,\omega_i).
\label{eq:matchedrun}
\end{equation}
The records $O_i^-$ and $O_i^+$ link observations to their implementations and execution attempts. Both runs use the validation rule specified in $s_i$. Matching conditions controls data and execution differences; training randomness also affects performance comparisons~\citep{bosma2024reproducibility}.

For valid executions with finite measurements and correctly matched reports, let $m_i^-$ and $m_i^+$ be the control and treatment metrics, and let $d_q\in\{-1,1\}$ denote minimization or maximization. The numerical test criterion is expressed by the predeclared predicate $\phi_i(y)=\mathbf 1[y\ \bowtie_i\ \theta_i]$, with $\bowtie_i\in\{<,\le,=,\ge,>\}$ and threshold $\theta_i$. The resulting criterion separates performance tests from property checks according to the declared experimental objective:
\begin{equation}
\begin{aligned}
\delta_i&=d_q(m_i^+-m_i^-),\\
g_i&=\begin{cases}
\mathbf 1[\delta_i>\max(0,\epsilon_i)]\,\phi_i(y_i^{\rm test}),
&\text{performance test},\\
\phi_i(a_i^{\rm obs}),&\text{property check}.
\end{cases}
\end{aligned}
\label{eq:scalarpredicate}
\end{equation}
A performance test requires improvement beyond zero and the protocol threshold $\epsilon_i$, and applies $\phi_i$ to $y_i^{\rm test}\in\{\delta_i,m_i^+\}$ as specified before execution. A property check applies $\phi_i$ to the reported intermediate observable $a_i^{\rm obs}$. An observable such as the number of failed property checks can summarize several assertions. Appendix~\ref{app:details} gives the full numerical and report-matching checks.

Execution validity and effect assessment are separate. A checkpoint is \emph{qualified} for retention after successful execution, implementation review, task checks, and a passed test, with no pending plan revision. A component is \emph{accepted} when its search ends with a retained qualified checkpoint. Failed execution triggers repair; a valid test that fails its criterion informs revision of the expected effect or design. Appendices~\ref{app:details} and~\ref{app:states} provide further details.

\subsection{Research Memory and Revision}
\label{sec:memory}
\label{sec:revision}
For each component, DAGLoop records the specification version, prerequisites, checkpoints, code changes, observations, validation results, and execution artifacts. Each observation remains associated with the decision, expected effect, conditions, and specification version tested. These records supply the evidence for subsequent implementation and method revision.

At research step $t$, let $\mathcal L_t(j)$ record component $j$'s history. For active component $i$, the agent receives
\begin{equation}
\mathcal H_t(i)=\left(
\{\mathcal L_t(j):j\in K_i\cup\{i\}\},\quad
\{\operatorname{Summary}(\mathcal L_t(j)):j\notin K_i\cup\{i\}\}
\right).
\label{eq:memoryprojection}
\end{equation}
Detailed histories of the active component and its ancestors let the agent inspect how the control was built, which prerequisite checks passed, and how the current attempt differs from earlier ones. Summaries of other branches supply context for revising the method without including their full execution traces in each component's input. Failed attempts remain available for debugging; method revision receives the accumulated research memory.

For example, when a valid test of $z_2$ does not support its expected gain, the record identifies the control containing $z_1$. Revision can reconsider the objective while retaining the supported representation. If revision changes $z_1$, the earlier result for $z_2$ remains a record of the previous experiment. The revised representation defines a new control for the objective; testing it again determines whether the expected effect persists after the prerequisite changes.

DAGLoop reuses a saved implementation only when its specification, prerequisite implementations, and recorded execution context remain compatible. Otherwise, the earlier result stays in memory for reference. Appendix~\ref{app:states} gives the equality and checkpoint checks used for reuse.

Search processes an active component first; a pending component becomes ready after its prerequisites are accepted. The agent drafts the first implementation, debugs execution errors or code that does not follow the specification, and reassesses the expectation when a valid test fails its criterion. Otherwise, it improves the implementation. Each operation uses research memory. Replanning can revise the decomposition or group decisions within the revision budget. When no component is ready, the system reports which prerequisites remain unresolved.

Component search retains the best qualified checkpoint. Performance tests rank checkpoints by the task metric; property checks use the specified comparison rule, with task performance breaking ties. Search ends when its attempt budget is exhausted or the workflow decides to \mbox{advance with a qualified checkpoint (Appendix~\ref{app:states}).}

\subsection{Complete-Candidate Evaluation}
\label{sec:algorithm}
Once all components are accepted, DAGLoop assembles their checkpoints and runs the official evaluator. Targeted tests assess effects in the prerequisite context, whereas the assembled candidate also includes decisions from other branches. These combinations can change task performance. Complete-candidate evaluation checks whether the supported changes improve the full method, and its result enters research memory to guide the next revision. Only complete candidates that pass the task checks can replace the current best:
\begin{equation}
\widehat M_{t+1}\in
\arg\max_{M\in\{\widehat M_t\}\cup\mathcal C_t^{\rm valid}}J_q(M).
\label{eq:incumbent}
\end{equation}
Here $\mathcal C_t^{\rm valid}$ contains newly evaluated complete candidates. Component measurements guide revision through the structured research memory. Algorithm~\ref{alg:loop} summarizes the loop; Appendices~\ref{app:details} and~\ref{app:states} detail the procedure, state updates, and evidence reuse.

\begin{algorithm}[htbp]
\caption{DAGLoop: targeted experiments and revision}
\label{alg:loop}
\begin{algorithmic}[1]
\Require Task, baseline $M_0$, proposed candidate $M$, budget $B$
\State Initialize research memory $\mathcal L$ and best candidate $\widehat M\gets M_0$
\While{budget and search limits permit}
  \State Construct or revise specifications and the dependency DAG using $(M,\mathcal L)$
  \For{each ready component $i$ in topological order}
    \State Search component $i$: build $P_i$, implement and test $T_i$ under the declared protocol
    \State Record validity and effect evidence in $\mathcal L$; debug or revise using this memory
    \State Retain the best qualified checkpoint when component search ends
  \EndFor
  \State If all components are accepted, assemble and evaluate $M^+$; update $\widehat M$
  \State Revise $M$ from $\mathcal L$ if further search is permitted; otherwise stop
\EndWhile
\State \Return $\widehat M,\mathcal L$
\end{algorithmic}
\end{algorithm}

\section{Related Work}

\paragraph{Automated AI Research.}
Recent work automates method proposal, implementation, experimentation, and revision. MLAgentBench studies language agents that conduct machine-learning experiments; The AI Scientist, its extension using agentic tree search, and Agent Laboratory address longer research workflows \citep{huang2023mlagentbench,lu2024aiscientist,yamada2025aiscientistv2,lu2026endtoend,schmidgall2025agentlab}. AlphaEvolve uses feedback from executable evaluators to guide the discovery of programs and algorithms \citep{novikov2025alphaevolve}.

Executable benchmarks make this research loop measurable. MLE-bench, MLGym, and RE-Bench evaluate agents on executable tasks; MLR-Bench assesses open-ended research \citep{chan2024mlebench,nathani2025mlgym,wijk2024rebench,chen2025mlrbench}. ScienceAgentBench evaluates scientific code using execution results and cost \citep{chen2025scienceagentbench}. ResearchGym, AutoLab, and AutoResearchEval examine longer research trajectories, artifact improvement, and failures linking experiments to claims \citep{garikaparthi2026researchgym,xu2026autolab,fei2026autoresearcheval}. Our work studies the evidence used to revise design decisions within these executable research settings.

\paragraph{Evaluator Feedback and Research Memory.}
OPRO conditions candidate generation on score histories; FunSearch selects evaluated programs for further search; Automated Design of Agentic Systems (ADAS) and AFlow optimize agent designs and workflows \citep{yang2024opro,romeraparedes2024funsearch,hu2025adas,zhang2025aflow}. MetaReflection uses past reflections; TextGrad and Trace optimize variables using textual and execution feedback \citep{gupta2024metareflection,yuksekgonul2025textgrad,cheng2024trace}. The AI Scientist maintains an experimental journal, Meta-Harness exposes earlier candidates' code, scores, and traces, and ResearchLoop organizes questions, evidence, and claims as durable research state \citep{lu2026endtoend,lee2026metaharness,xia2026researchloop}. AutoResearchEval examines links between implementations, observations, and claims \citep{fei2026autoresearcheval}. Our memory associates each outcome with the tested decision, expected effect, and prerequisite context, so revision can use evidence at the level of the decision being changed.

\paragraph{Ablation Planning and Experimental Attribution.}
AbGen and AblationBench evaluate whether language models can design informative ablations, and The AI Scientist includes ablations in its experiment schedule \citep{zhao2025abgen,abramovich2025ablationbench,lu2026endtoend}. Software composition also requires explicit component dependencies. Cordis represents these dependencies, motivating our description of the properties that one decision requires from another \citep{shi2026cordis}. DAGLoop uses these dependencies to construct the control implementation for each targeted experiment during revision. It distinguishes task improvement from a property required by another decision and uses joint tests when an expected effect requires several decisions. The resulting evidence guides decision retention, revision, and further testing.

\section{Experiments}
\label{sec:experiments}
We evaluate DAGLoop on diverse executable research tasks. We use AutoLab as the main benchmark and compare with locally run Terminus-2 under four model and reasoning settings. Results from the additional FML-Bench evaluation appear in Appendix~\ref{app:fmlresults}.

\subsection{Benchmark and Task Protocol}
\label{sec:benchmark}
We use all 36 official AutoLab tasks: 25 CPU tasks and 11 GPU tasks covering CPU optimization, CUDA kernels, program synthesis, and model development \citep{xu2026autolab}. Each task provides a baseline repository, evaluator, reward rule, correctness checks, and protected task files. DAGLoop reproduces the baseline locally, proposes a candidate, constructs the dependency DAG, runs component experiments, and evaluates candidate checkpoints.

Both agents are evaluated with GLM-5.2 at maximum reasoning effort and with reasoning disabled, and with DeepSeek-V4-Flash and DeepSeek-V4.1-Flash with reasoning disabled. The GLM-5.2 maximum-reasoning GPU runs use NVIDIA H800 GPUs. The comparisons use the official Terminus-2 agent and original task verifiers~\citep{harborterminus2}. The GLM-5.2 and DeepSeek-V4-Flash references use Harbor~0.3.0. Agent wall-clock limits are disabled, while task-specific training and evaluation constraints remain in force for both agents.

\label{sec:auditpolicy}
All Terminus-2 task scores across the four model/reasoning settings are arithmetic means of two runs per task. DAGLoop averages two runs per task in the three reasoning-disabled settings and uses one run per task with GLM-5.2 at maximum reasoning effort. Section~\ref{sec:algorithm} describes how the system selects its final complete candidate. Appendix~\ref{app:fullresults} lists all task scores and aggregation rules.

\label{sec:metrics}

Higher scores are better. We report task scores and means for the full suite, CPU tasks, and GPU tasks. Reward thresholds and clipping can produce ties despite different native metrics. Paired agent comparisons use task scores. Win, tie, and loss counts compare scores rounded to six decimal places.

\subsection{Overall, Per-Task, and Component Ablation Results}
\label{sec:results}
\begin{table}[!htbp]
\centering
\caption{Mean AutoLab scores for eight agent configurations.}
\label{tab:breadthsummary}
\normalsize
\setlength{\tabcolsep}{4pt}
\begin{tabular*}{\linewidth}{@{\extracolsep{\fill}}lllrrr@{}}
\toprule
Agent & Model & Reasoning & All (36) & CPU (25) & GPU (11) \\
\midrule
\textbf{DAGLoop} & GLM-5.2 & Max & \textbf{0.6459} & \textbf{0.6539} & \textbf{0.6277} \\
Terminus-2 & GLM-5.2 & Max & 0.5831 & 0.6247 & 0.4885 \\
\textbf{DAGLoop} & GLM-5.2 & Off & \textbf{0.5430} & \textbf{0.5675} & \textbf{0.4875} \\
Terminus-2 & GLM-5.2 & Off & 0.4266 & 0.4502 & 0.3730 \\
\textbf{DAGLoop} & DeepSeek-V4-Flash & Off & \textbf{0.5258} & \textbf{0.5137} & \textbf{0.5535} \\
Terminus-2 & DeepSeek-V4-Flash & Off & 0.4202 & 0.4492 & 0.3542 \\
\textbf{DAGLoop} & DeepSeek-V4.1-Flash & Off & \textbf{0.5819} & \textbf{0.5711} & \textbf{0.6063} \\
Terminus-2 & DeepSeek-V4.1-Flash & Off & 0.4926 & 0.5256 & 0.4176 \\
\bottomrule
\end{tabular*}
\end{table}

DAGLoop has higher mean scores than Terminus-2 on the full suite and in both resource groups for all comparisons in Table~\ref{tab:breadthsummary}. With GLM-5.2 at maximum reasoning effort, the full-suite means are 0.6459 and 0.5831. With reasoning disabled, the means are 0.5430 and 0.4266 for GLM-5.2, 0.5258 and 0.4202 for DeepSeek-V4-Flash, and 0.5819 and 0.4926 for DeepSeek-V4.1-Flash. Each mean includes every task in its split. In the DeepSeek-V4.1-Flash comparison with reasoning disabled, DAGLoop scores higher on 26 tasks and lower on ten, with a mean advantage of 0.0892. Table~\ref{tab:componentablations} traces DeepSeek-V4.1-Flash component experiments. Appendix~\ref{app:fullresults} lists all task scores.

With GLM-5.2 Max, DAGLoop scores 0.7378 versus 0.0006 on online language-model serving, 0.4637 versus 0.2403 on agent tool routing, and 0.8134 versus 0.6040 on concurrent key-value storage. Terminus-2 leads on instruction-following data selection (0.8581 versus 0.6177) and sorted string table (SSTable) compaction (0.3479 versus 0.1780). Five tasks tie at 1.0000: BM25, multilingual optical character recognition, stack-machine golf, Toy ISA, and VLIW scheduling.

\looseness=-1
We evaluate each decision against a matched control that preserves its declared prerequisite implementation. In Table~\ref{tab:componentablations}, the Parent column is this control and Observed is the candidate metric or property-check value. Direct stages compare the task metric, property stages test an intermediate requirement such as zero failed checks, and integration stages evaluate components jointly when their effect depends on their combination. The table traces experimental evidence along dependency chains.

\begin{table}[H]
\centering
\caption{Five component trajectories with DeepSeek-V4.1-Flash and reasoning disabled. Each row is one saved observation. Direct stages retain their prerequisite checkpoints; property stages report their declared check. For routing, excess candidates is the mean positive-score candidate count minus 3,600. Integration rows ($\dagger$) report official complete-candidate measurements. Property-check and integration rows report only the observed result; dashes mark Parent and $\Delta$.}
\label{tab:componentablations}
\normalsize
\setlength{\tabcolsep}{1.8pt}
\begin{tabular*}{\linewidth}{@{\extracolsep{\fill}}l>{\raggedright\arraybackslash}p{90pt}rrlr@{}}
\toprule
Task & Stage / component & Parent & Observed & Metric/check & $\Delta$ \\
\midrule
BVH raytracer & C1 direct: SAH BVH traversal & 16.6515 & 0.2111 & runtime (s) $\downarrow$ & $-16.4405$ \\
BVH raytracer & C2 direct: BVH4 SIMD & 0.2266 & 0.1479 & runtime (s) $\downarrow$ & $-0.0787$ \\
BVH raytracer & Integration$^\dagger$: C1--C2 & -- & 0.1553 & runtime (s) $\downarrow$ & -- \\
Adversarial splay & C1 property: splay simulator & -- & 0 & mismatches $=0$ & -- \\
Adversarial splay & C2 direct: sequence search & 48,656 & 69,665 & rotations $\uparrow$ & $+21,009$ \\
Adversarial splay & Integration$^\dagger$: C1--C2 & -- & 69,665 & rotations $\uparrow$ & -- \\
AES-128 CTR & C1 direct: 8-way AES-NI CTR & 7.9578 & 0.0468 & runtime (s) $\downarrow$ & $-7.9110$ \\
AES-128 CTR & C2 direct: unroll and prefetch & 0.0467 & 0.0459 & runtime (s) $\downarrow$ & $-0.0009$ \\
AES-128 CTR & Integration$^\dagger$: C1--C2 & -- & 0.0454 & runtime (s) $\downarrow$ & -- \\
Hash join & C1 property: flat-hash build & -- & 0 & failed checks $=0$ & -- \\
Hash join & C2 direct: streaming probe & 74.8238 & 0.0383 & runtime (s) $\downarrow$ & $-74.7855$ \\
Hash join & Integration$^\dagger$: C1--C2 & -- & 0.0376 & runtime (s) $\downarrow$ & -- \\
Agent tool routing & C1 property: posting index & -- & $-2092.9425$ & excess cand. $<0$ & -- \\
Agent tool routing & C2 direct: inverted retrieval & 9.3053 & 0.9274 & runtime (s) $\downarrow$ & $-8.3778$ \\
Agent tool routing & Integration$^\dagger$: C1--C2 & -- & 0.9305 & runtime (s) $\downarrow$ & -- \\
\bottomrule
\end{tabular*}
\end{table}

\subsection{System Ablations}
\label{sec:systemablations}
We compare four research policies on six AutoLab tasks using DeepSeek-V4.1-Flash with reasoning disabled. All four reuse the same completed literature retrieval and baseline reproduction for each task. Starting from that fixed preparation snapshot, each policy runs the subsequent experimentation and method-revision workflow twice; each entry averages the two final scores. The shared preparation is not rerun. A is full DAGLoop. B tests components independently from a common baseline without retaining ancestor changes or their acceptance prerequisites. C replaces the model-facing structured memory view with chronological observations, retaining DAG experiments and the underlying evidence records. A and C use the same model context window and memory-presentation length ceiling; their actual histories and prompt lengths can differ. D retains the planner and critic but uses complete-candidate evaluation and revision without intermediate component experiments or property checks. Full-system scores come from new runs in a separate ablation cohort.

\begin{table}[!htbp]
\centering
\caption{System ablations with DeepSeek-V4.1-Flash and reasoning disabled. Each entry averages two runs from a shared pre-experiment preparation snapshot; higher is better. Full-system runs form a separate cohort from the runs used in the main benchmark.}
\label{tab:systemablations}
\normalsize
\setlength{\tabcolsep}{3pt}
\begin{tabular*}{\linewidth}{@{\extracolsep{\fill}}lrrrr@{}}
\toprule
Task & A: Full & B: Independent & C: History & D: Whole-only \\
\midrule
Adaptive Compression & \textbf{0.6042} & 0.5638 & 0.5852 & 0.5495 \\
AES-128 CTR & \textbf{0.6381} & 0.6120 & 0.6180 & 0.6029 \\
Hash Join & \textbf{0.5031} & 0.4663 & 0.4884 & 0.4566 \\
Flash Attention & \textbf{0.4443} & 0.4067 & 0.4430 & 0.3968 \\
NTT CUDA & \textbf{0.5417} & 0.4881 & 0.5002 & 0.4864 \\
MSM CUDA & \textbf{0.4931} & 0.4647 & 0.4695 & 0.4542 \\
\midrule
Mean (6) & \textbf{0.5374} & 0.5002 & 0.5173 & 0.4911 \\
\bottomrule
\end{tabular*}
\end{table}

Table~\ref{tab:systemablations} gives means of 0.5374, 0.5002, 0.5173, and 0.4911 for A--D. Relative to the full system, independent component experiments reduce the mean by 0.0372, chronological history by 0.0201, and whole-candidate-only feedback by 0.0464. The full system has the highest mean on each of the six tasks. These policy comparisons support prerequisite-preserving experiments, structured context, and intermediate feedback on the evaluated six-task subset.

\section{Limitations}
DAGLoop's local evidence depends on the component decomposition, declared prerequisites, and property checks. Additional dependencies or interactions may call for revised specifications or joint tests; a joint result supports the tested group of decisions.

Evidence reuse is tied to the implementation and evaluation conditions recorded with each result. Changes to upstream components, data, or hardware may require \mbox{revalidating affected comparisons, adding to the cost of revision.}

\section{Conclusion}
We present DAGLoop as a general-purpose research system for executable method improvement. It connects candidate-method revision to targeted experiments on its constituent design decisions. Experimental specifications link each decision's rationale and expected effect to tests that preserve its prerequisites. Structured research memory carries the outcomes into subsequent revisions, while complete-candidate evaluation measures task performance. Preserving prerequisites lets the agent assess intermediate properties and downstream effects in their intended context. Joint experiments assess effects that require several decisions to act together, while the accumulated records guide further tests when the evidence remains unresolved. On 36 AutoLab tasks, DAGLoop obtains a mean task score of 0.6459, compared with 0.5831 for Terminus-2 using GLM-5.2 with maximum reasoning effort. DAGLoop scores higher on 19 tasks, ties on seven, and scores lower on ten. DAGLoop also has higher mean scores in all three reasoning-disabled settings, each averaging two runs per task.

\clearpage
\section*{Ethics Statement}
Automated research consumes compute and can amplify errors in experimental interpretation. DAGLoop retains evaluator requirements, experiment records, and the context of each result. The authors remain responsible for resource use, interpretation, and reported claims.

\section*{AI Use Statement}
Generative AI tools were used for literature search and screening, method and experiment planning, code development and research execution, assistance with mathematical proofs, data analysis and figure preparation, and manuscript drafting, translation, and language editing. The authors review the text, citations, and experimental claims and take responsibility for the final work.

\bibliographystyle{iclr2027_conference}
\bibliography{references}

@inproceedings{huang2023mlagentbench,
  title={{MLA}gent{B}ench: Evaluating Language Agents on Machine Learning Experimentation},
  author={Huang, Qian and Vora, Jian and Liang, Percy and Leskovec, Jure},
  booktitle={Proceedings of the 41st International Conference on Machine Learning},
  volume={235},
  pages={20271--20309},
  publisher={PMLR},
  year={2024},
  url={https://proceedings.mlr.press/v235/huang24y.html}
}

@misc{lu2024aiscientist,
  title={The {AI} Scientist: Towards Fully Automated Open-Ended Scientific Discovery},
  author={Chris Lu and Cong Lu and Robert Tjarko Lange and Jakob Foerster and Jeff Clune and David Ha},
  year={2024},
  eprint={2408.06292},
  archivePrefix={arXiv},
  primaryClass={cs.AI},
  url={https://arxiv.org/abs/2408.06292}
}

@misc{yamada2025aiscientistv2,
  title={The {AI} Scientist-v2: Workshop-Level Automated Scientific Discovery via Agentic Tree Search},
  author={Yutaro Yamada and Robert Tjarko Lange and Cong Lu and Shengran Hu and Chris Lu and Jakob Foerster and Jeff Clune and David Ha},
  year={2025},
  eprint={2504.08066},
  archivePrefix={arXiv},
  primaryClass={cs.AI},
  url={https://arxiv.org/abs/2504.08066}
}

@misc{novikov2025alphaevolve,
  title={{AlphaEvolve}: A coding agent for scientific and algorithmic discovery},
  author={Alexander Novikov and Ng{\^{a}}n V{\~{u}} and Marvin Eisenberger and Emilien Dupont and Po-Sen Huang and Adam Zsolt Wagner and Sergey Shirobokov and Borislav Kozlovskii and Francisco J. R. Ruiz and Abbas Mehrabian and M. Pawan Kumar and Abigail See and Swarat Chaudhuri and George Holland and Alex Davies and Sebastian Nowozin and Pushmeet Kohli and Matej Balog},
  year={2025},
  eprint={2506.13131},
  archivePrefix={arXiv},
  primaryClass={cs.AI},
  url={https://arxiv.org/abs/2506.13131}
}

@inproceedings{chan2024mlebench,
  title={{MLE-bench}: Evaluating Machine Learning Agents on Machine Learning Engineering},
  author={Chan, Jun Shern and Chowdhury, Neil and Jaffe, Oliver and Aung, James and Sherburn, Dane and Mays, Evan and Starace, Giulio and Liu, Kevin and Maksin, Leon and Patwardhan, Tejal and Madry, Aleksander and Weng, Lilian},
  booktitle={International Conference on Learning Representations},
  volume={2025},
  pages={50466--50494},
  year={2025},
  url={https://proceedings.iclr.cc/paper_files/paper/2025/file/7e3767db483c942b883eb4f8cfb74e31-Paper-Conference.pdf}
}

@misc{wijk2024rebench,
  title={{RE-Bench}: Evaluating frontier {AI} R\&D capabilities of language model agents against human experts},
  author={Hjalmar Wijk and Tao Lin and Joel Becker and Sami Jawhar and Neev Parikh and Thomas Broadley and Lawrence Chan and Michael Chen and Josh Clymer and Jai Dhyani and Elena Ericheva and Katharyn Garcia and Brian Goodrich and Nikola Jurkovic and Holden Karnofsky and Megan Kinniment and Aron Lajko and Seraphina Nix and Lucas Sato and William Saunders and Maksym Taran and Ben West and Elizabeth Barnes},
  year={2025},
  eprint={2411.15114},
  archivePrefix={arXiv},
  primaryClass={cs.LG},
  url={https://arxiv.org/abs/2411.15114}
}

@misc{nathani2025mlgym,
  title={{MLGym}: A New Framework and Benchmark for Advancing {AI} Research Agents},
  author={Deepak Nathani and Lovish Madaan and Nicholas Roberts and Nikolay Bashlykov and Ajay Menon and Vincent Moens and Amar Budhiraja and Despoina Magka and Vladislav Vorotilov and Gaurav Chaurasia and Dieuwke Hupkes and Ricardo Silveira Cabral and Tatiana Shavrina and Jakob Foerster and Yoram Bachrach and William Yang Wang and Roberta Raileanu},
  year={2025},
  eprint={2502.14499},
  archivePrefix={arXiv},
  primaryClass={cs.CL},
  url={https://arxiv.org/abs/2502.14499}
}

@inproceedings{chen2025mlrbench,
  title={{MLR-Bench}: Evaluating {AI} Agents on Open-Ended Machine Learning Research},
  author={Chen, Hui and Xiong, Miao and Lu, Yujie and Han, Wei and Deng, Ailin and He, Yufei and Wu, Jiaying and Li, Yibo and Liu, Yue and Hooi, Bryan},
  booktitle={Advances in Neural Information Processing Systems},
  volume={38, Main Conference},
  publisher={Curran Associates, Inc.},
  year={2025},
  doi={10.52202/085713-3945},
  url={https://proceedings.neurips.cc/paper_files/paper/2025/file/ab8dd000d6f87f40061a73f8bca7fae4-Paper-Datasets_and_Benchmarks_Track.pdf}
}

@misc{xu2026autolab,
  title={{AutoLab}: Can Frontier Models Solve Long-Horizon Auto Research and Engineering Tasks?},
  author={Zhangchen Xu and Junda Chen and Yue Huang and Dongfu Jiang and Jiefeng Chen and Hang Hua and Zijian Wu and Zheyuan Liu and Zexue He and Lichi Li and Shizhe Diao and Jiaxin Pei and Jinsung Yoon and Hao Zhang and Mengdi Wang and Radha Poovendran and Misha Sra and Alex Pentland and Zichen Chen},
  year={2026},
  eprint={2606.05080},
  archivePrefix={arXiv},
  primaryClass={cs.AI},
  url={https://arxiv.org/abs/2606.05080}
}

@misc{zou2026fmlbench,
  title={FML-bench: A Controlled Study of {AI} Research Agent Strategies from the Perspective of Search Dynamics},
  author={Qiran Zou and Hou Hei Lam and Wenhao Zhao and Tingting Chen and Yiming Tang and Samson Yu and Yingtao Zhu and Srinivas Anumasa and Zufeng Zhang and Tianyi Zhang and Chang Liu and Zhengyao Jiang and Anirudh Goyal and Dianbo Liu},
  year={2026},
  eprint={2605.17373},
  archivePrefix={arXiv},
  primaryClass={cs.LG},
  url={https://arxiv.org/abs/2605.17373}
}

@misc{garikaparthi2026researchgym,
  title={{ResearchGym}: Evaluating Language Model Agents on Real-World {AI} Research},
  author={Aniketh Garikaparthi and Manasi Patwardhan and Arman Cohan},
  year={2026},
  eprint={2602.15112},
  archivePrefix={arXiv},
  primaryClass={cs.AI},
  url={https://arxiv.org/abs/2602.15112}
}

@misc{fei2026autoresearcheval,
  title={How Do Agents Fail on {AutoResearch}: End-to-End Diagnostic Evaluation on 100 Real-World Frontier Research Tasks},
  author={Yanlin Fei and Nazhou Liu and Xinmiao Yu and Shaolong Chen and Lei Li and Rahul Thapa and Madalina Ciobanu and Navan Preet Singh and Qingqing Mao and Ritankar Das},
  year={2026},
  eprint={2608.14905},
  archivePrefix={arXiv},
  primaryClass={cs.CL},
  url={https://arxiv.org/abs/2608.14905v3}
}

@inproceedings{yang2024opro,
  title={Large Language Models as Optimizers},
  author={Yang, Chengrun and Wang, Xuezhi and Lu, Yifeng and Liu, Hanxiao and Le, Quoc V and Zhou, Denny and Chen, Xinyun},
  booktitle={International Conference on Learning Representations},
  volume={2024},
  pages={12028--12068},
  year={2024},
  url={https://proceedings.iclr.cc/paper_files/paper/2024/file/3339f19c5fcee3ad74502947a32be9e6-Paper-Conference.pdf}
}

@article{romeraparedes2024funsearch,
  title={Mathematical discoveries from program search with large language models},
  author={Romera-Paredes, Bernardino and Barekatain, Mohammadamin and Novikov, Alexander and Balog, Matej and Kumar, M. Pawan and Dupont, Emilien and Ruiz, Francisco J. R. and Ellenberg, Jordan S. and Wang, Pengming and Fawzi, Omar and Kohli, Pushmeet and Fawzi, Alhussein},
  journal={Nature},
  volume={625},
  number={7995},
  pages={468--475},
  year={2024},
  doi={10.1038/s41586-023-06924-6},
  url={https://doi.org/10.1038/s41586-023-06924-6}
}

@inproceedings{hu2025adas,
  title={Automated Design of Agentic Systems},
  author={Hu, Shengran and Lu, Cong and Clune, Jeff},
  booktitle={International Conference on Learning Representations},
  volume={2025},
  pages={21344--21377},
  year={2025},
  url={https://proceedings.iclr.cc/paper_files/paper/2025/file/36b7acf6f6010652b3f2a433774a66fe-Paper-Conference.pdf}
}

@inproceedings{zhang2025aflow,
  title={{AFlow}: Automating Agentic Workflow Generation},
  author={Zhang, Jiayi and Xiang, Jinyu and Yu, Zhaoyang and Teng, Fengwei and Chen, XiongHui and Chen, Jiaqi and Zhuge, Mingchen and Cheng, Xin and Hong, Sirui and Wang, Jinlin and Zheng, Bingnan and Liu, Bang and Luo, Yuyu and Wu, Chenglin},
  booktitle={International Conference on Learning Representations},
  volume={2025},
  pages={34040--34077},
  year={2025},
  url={https://proceedings.iclr.cc/paper_files/paper/2025/file/5492ecbce4439401798dcd2c90be94cd-Paper-Conference.pdf}
}

@inproceedings{cheng2024trace,
  title={{Trace} is the Next {AutoDiff}: Generative Optimization with Rich Feedback, Execution Traces, and {LLMs}},
  author={Cheng, Ching-An and Nie, Allen and Swaminathan, Adith},
  booktitle={Advances in Neural Information Processing Systems},
  volume={37},
  pages={71596--71642},
  publisher={Curran Associates, Inc.},
  year={2024},
  doi={10.52202/079017-2287},
  url={https://proceedings.neurips.cc/paper_files/paper/2024/file/83ba7056bce2c3c3c27e17397cf3e1f0-Paper-Conference.pdf}
}

@misc{lee2026metaharness,
  title={{Meta-Harness}: End-to-End Optimization of Model Harnesses},
  author={Yoonho Lee and Roshen Nair and Qizheng Zhang and Kangwook Lee and Omar Khattab and Chelsea Finn},
  year={2026},
  eprint={2603.28052},
  archivePrefix={arXiv},
  primaryClass={cs.AI},
  url={https://arxiv.org/abs/2603.28052}
}

@misc{abramovich2025ablationbench,
  title={{AblationBench}: Evaluating Automated Planning of Ablations in Empirical {AI} Research},
  author={Talor Abramovich and Gal Chechik},
  year={2026},
  eprint={2507.08038},
  archivePrefix={arXiv},
  primaryClass={cs.CL},
  url={https://arxiv.org/abs/2507.08038}
}

@misc{xia2026researchloop,
  title={{ResearchLoop}: An Evidence-Gated Control Plane for {AI}-Assisted Research},
  author={Yihan Xia and Taotao Wang},
  year={2026},
  eprint={2605.28282},
  archivePrefix={arXiv},
  primaryClass={cs.AI},
  url={https://arxiv.org/abs/2605.28282}
}

@misc{shi2026cordis,
  title={A Programming Paradigm for Spatiotemporal Composability},
  author={Yifan Shi and Wei Zhang and Tianyi Cui},
  year={2026},
  eprint={2608.25512},
  archivePrefix={arXiv},
  primaryClass={cs.PL},
  url={https://arxiv.org/abs/2608.25512}
}

@article{lu2026endtoend,
  title={Towards end-to-end automation of {AI} research},
  author={Lu, Chris and Lu, Cong and Lange, Robert Tjarko and Yamada, Yutaro and Hu, Shengran and Foerster, Jakob and Ha, David and Clune, Jeff},
  journal={Nature},
  volume={651},
  number={8107},
  pages={914--919},
  year={2026},
  doi={10.1038/s41586-026-10265-5},
  url={https://doi.org/10.1038/s41586-026-10265-5}
}

@inproceedings{schmidgall2025agentlab,
  title={Agent Laboratory: Using {LLM} Agents as Research Assistants},
  author={Schmidgall, Samuel and Su, Yusheng and Wang, Ze and Sun, Ximeng and Wu, Jialian and Yu, Xiaodong and Liu, Jiang and Moor, Michael and Liu, Zicheng and Barsoum, Emad},
  booktitle={Findings of the Association for Computational Linguistics: EMNLP 2025},
  pages={5977--6043},
  publisher={Association for Computational Linguistics},
  year={2025},
  doi={10.18653/v1/2025.findings-emnlp.320},
  url={https://aclanthology.org/2025.findings-emnlp.320/}
}

@inproceedings{chen2025scienceagentbench,
  title={{ScienceAgentBench}: Toward Rigorous Assessment of Language Agents for Data-Driven Scientific Discovery},
  author={Chen, Ziru and Chen, Shijie and Ning, Yuting and Zhang, Qianheng and Wang, Boshi and Yu, Botao and Li, Yifei and Liao, Zeyi and Wei, Chen and Lu, Zitong and Dey, Vishal and Xue, Mingyi and Baker, Frazier N. and Burns, Benjamin and Adu-Ampratwum, Daniel and Huang, Xuhui and Ning, Xia and Gao, Song and Su, Yu and Sun, Huan},
  booktitle={International Conference on Learning Representations},
  volume={2025},
  pages={96934--96990},
  year={2025},
  url={https://proceedings.iclr.cc/paper_files/paper/2025/file/f12b4df26344f3be803c06b555252efe-Paper-Conference.pdf}
}

@inproceedings{zhao2025abgen,
  title={{A}b{G}en: Evaluating Large Language Models in Ablation Study Design and Evaluation for Scientific Research},
  author={Zhao, Yilun and Chen, Weiyuan and Xu, Zhijian and Patwardhan, Manasi and Wang, Chengye and Liu, Yixin and Vig, Lovekesh and Cohan, Arman},
  booktitle={Proceedings of the 63rd Annual Meeting of the Association for Computational Linguistics (Volume 1: Long Papers)},
  pages={12479--12491},
  publisher={Association for Computational Linguistics},
  year={2025},
  doi={10.18653/v1/2025.acl-long.611},
  url={https://aclanthology.org/2025.acl-long.611/}
}

@inproceedings{gupta2024metareflection,
  title={{M}eta{R}eflection: Learning Instructions for Language Agents using Past Reflections},
  author={Gupta, Priyanshu and Kirtania, Shashank and Singha, Ananya and Gulwani, Sumit and Radhakrishna, Arjun and Soares, Gustavo and Shi, Sherry},
  booktitle={Proceedings of the 2024 Conference on Empirical Methods in Natural Language Processing},
  pages={8369--8385},
  publisher={Association for Computational Linguistics},
  year={2024},
  doi={10.18653/v1/2024.emnlp-main.477},
  url={https://aclanthology.org/2024.emnlp-main.477/}
}

@article{yuksekgonul2025textgrad,
  title={Optimizing generative {AI} by backpropagating language model feedback},
  author={Yuksekgonul, Mert and Bianchi, Federico and Boen, Joseph and Liu, Sheng and Lu, Pan and Huang, Zhi and Guestrin, Carlos and Zou, James},
  journal={Nature},
  volume={639},
  number={8055},
  pages={609--616},
  year={2025},
  doi={10.1038/s41586-025-08661-4},
  url={https://doi.org/10.1038/s41586-025-08661-4}
}

@inproceedings{bosma2024reproducibility,
  title={Reproducibility of Training Deep Learning Models for Medical Image Analysis},
  author={Bosma, Joeran Sander and Peeters, Dr\'e and Alves, Nat\'alia and Saha, Anindo and Saghir, Zaigham and Jacobs, Colin and Huisman, Henkjan},
  booktitle={Medical Imaging with Deep Learning},
  volume={227},
  pages={1269--1287},
  publisher={PMLR},
  year={2024},
  url={https://proceedings.mlr.press/v227/bosma24a.html}
}

@misc{harborterminus2,
  title={Terminus-2 in Harbor (version 0.3.0)},
  author={{Harbor Framework Team}},
  year={2026},
  url={https://github.com/harbor-framework/harbor/tree/1d97c20a1729f068a2bb95af6100f0f818ef0595},
  note={Official agent implementation, release v0.3.0; accessed September 26, 2026}
}

@misc{autolabrepo,
  title={{AutoLab}: Official Benchmark Repository},
  author={{AutoLab Authors}},
  year={2026},
  url={https://github.com/autolabhq/autolab/tree/4127da3dde8449be61a1cf9859473b9fbbd51751},
  note={Task definitions at commit 4127da3dde8449be61a1cf9859473b9fbbd51751}
}

\appendix

\section{Detailed Experimental Procedure}
\label{app:details}

The executable task and budgeted objective are defined in \cref{eq:task,eq:taskobjective}. The procedures below implement its targeted experiments and method revision.

\subsection{Graph Construction and Comparison Configurations}
Let $N$ convert component names to a consistent identifier format (Appendix~\ref{app:states}). The declared dependencies $d_i^0$ and ancestor sets $K_i$ are defined in \cref{eq:graphclosure}.

The constructor rejects duplicate identifiers, unknown or self-referential dependencies, and cycles. Kahn's topological sort, using a first-in, first-out (FIFO) queue for eligible nodes, gives the order $\pi$ in which the ancestor sets are computed, preserving the input traversal order when nodes become eligible together. The plan also declares $a_i$, the set of components that must be accepted before testing component $i$. The constructor removes redundant direct edges, retaining the parents $d_i$, and requires $a_i$ to equal all ancestors:
\begin{equation}
\begin{aligned}
d_i&=d_i^0\setminus\{j\in d_i^0:\exists k\in d_i^0\setminus\{j\},\ j\in K_k\},\\
a_i&=K_i,\qquad A=\{(j,i):j\in d_i\}.
\end{aligned}
\label{eq:reduceddependencies}
\end{equation}
The equality is a validation requirement: a missing ancestor or an unrelated component causes rejection. Semantic adequacy of the proposed dependencies and tests is assessed by the critic.

\begin{algorithm}[htbp]
\caption{Construction of the dependency DAG}
\label{alg:construct}
\begin{algorithmic}[1]
\Require Declared components $\{(i,d_i^0,a_i,s_i)\}_{i=1}^n$
\State Normalize identifiers with $N$; reject duplicates and invalid references
\State $A_0\gets\{(j,i):j\in d_i^0\}$; $\pi\gets\operatorname{FIFO\text{-}Kahn}(V,A_0)$
\If{$|\pi|\ne |V|$}
  \State \Return cycle error
\EndIf
\For{$i$ in $\pi$}
  \State $K_i\gets\bigcup_{j\in d_i^0}(\{j\}\cup K_j)$
\EndFor
\For{$i$ in $\pi$}
  \State $d_i\gets\{j\in d_i^0:\nexists k\in d_i^0\setminus\{j\},\ j\in K_k\}$
  \If{$a_i$ is missing or $a_i\ne K_i$}
    \State \Return prerequisite mismatch
  \EndIf
\EndFor
\State \Return $G=(V,\{(j,i):j\in d_i\}),\ \pi,\ \{K_i,s_i\}_{i\in V}$
\end{algorithmic}
\end{algorithm}

$\Build$ restores the checkpoints of the required components and combines their files with the baseline. When multiple branches contribute changes, it uses a shared earlier version as the merge base where available and reports conflicting edits. The treatment applies the active component's change to this control. Improvement can start from the component's previous best checkpoint while the control remains fixed. Unrelated branches are excluded. Implementation review checks whether the code follows the specification for the active component.

For a merged component $u$ representing decisions $S$, \cref{eq:parentcomparison} applies with $i=u$ and the shared intervention $\Delta_u$. The planner must supply an acyclic revised graph. For example, merging $z_a$ and $z_c$ in $z_a\to z_b\to z_c$ while retaining both incident dependencies gives $u\to z_b\to u$ and is rejected. The planner then revises the decomposition. \Cref{fig:cap_example} shows single-decision and joint comparisons.

\begin{figure}[htbp]
  \centering
  \includegraphics[width=\linewidth]{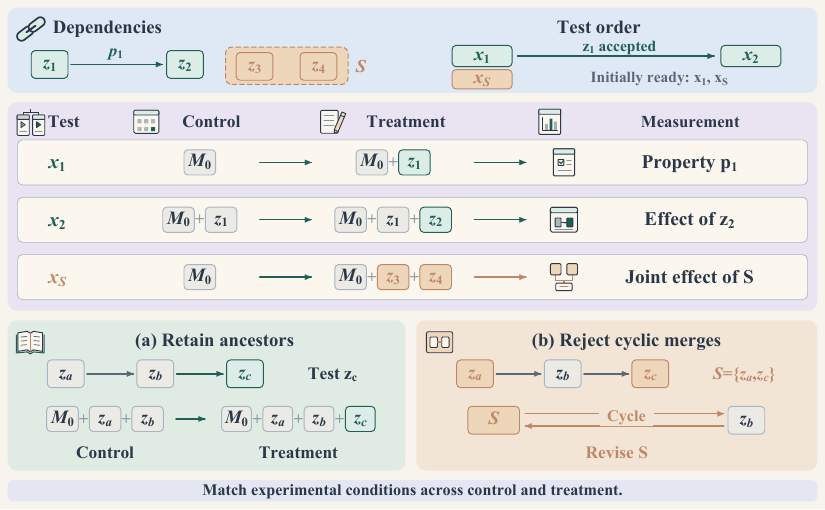}
  \caption{Comparisons that preserve required dependencies. The upstream test checks $p_1$; the downstream test includes $z_1$ in its control. Decisions in $S=\{z_3,z_4\}$ form one merged component with a shared test and result. A cyclic proposed merge is rejected during graph construction.}
  \label{fig:cap_example}
\end{figure}

\subsection{Numerical Criteria and Checkpoint Qualification}
$\mathsf{Fin}(v_1,\ldots,v_k)$ is true when every argument is a finite numerical value and none is Boolean. For a performance test, the predicate specified before execution chooses $y_i^{\rm test}=\delta_i$ or $m_i^+$, and the numerical acceptance criterion is
\begin{equation}
g_i^{\rm dir}=
\mathbf 1[\mathsf{Fin}(m_i^-,m_i^+)]\,
\mathbf 1[\delta_i>0]\,
\phi_i(y_i^{\rm test})\,
\mathbf 1[\delta_i>\epsilon_i].
\label{eq:directgate}
\end{equation}
Here $\epsilon_i$ is the minimum improvement specified by the measurement protocol, with $\epsilon_i=0$ for a test requiring positive improvement. Thus a positive observed gain must also satisfy the declared predicate. For a property check, let $a_i^{\rm obs}$ be its reported intermediate observable and $b_i=1$ indicate exactly one report matching the specified test and linked to its execution record. Its acceptance criterion is
\begin{equation}
g_i^{\rm ena}=b_i\,\mathbf 1[\mathsf{Fin}(a_i^{\rm obs})]\,\phi_i(a_i^{\rm obs}).
\label{eq:enablinggate}
\end{equation}
An aggregate observable, such as the number of failed property checks, can summarize several assertions. The execution script computes that aggregate; the criterion compares it with the threshold declared in the experimental specification before execution.

Numerical acceptance and checkpoint qualification are separate. For attempt $t$, let $e_{it}$ indicate successful execution, $h_{it}$ an existing checkpoint, $r_{it}$ a requested plan revision, and $c_{it}$ a review of task constraints that has not blocked the attempt. Method implementation review (MIR), denoted by $\mathrm{MIR}_{it}$, assesses whether the code implements the specified method. Let $g_{it}$ be $g_i^{\rm dir}$ or $g_i^{\rm ena}$ according to the test type. Qualification is
\begin{equation}
q_{it}=e_{it}h_{it}(1-r_{it})c_{it}\,
\mathbf 1[\mathrm{MIR}_{it}\in\{\mathrm{faithful},\mathrm{partial}\}]\,g_{it}.
\label{eq:qualified}
\end{equation}
A partial implementation supports only the implemented and tested scope. Failed attempts remain in the history, while an earlier qualified checkpoint can be retained.

\subsection{Memory Records and Feedback Scheduling}
The component histories record the specification reference, ancestor set, checkpoints, code difference, observations, validation results, and execution artifacts. We write an attempt record as
\begin{equation}
F_{it}=(z_i,\nu_i,e_i,D_i,P_i,O_{it}^-,O_{it}^+,g_{it},q_{it},\mathcal A_{it}),
\label{eq:evidence}
\end{equation}
where $\nu_i$ identifies the specification version and $\mathcal A_{it}$ references supporting artifacts.

Let $\sigma_t(i)\in\{\mathrm{pending},\mathrm{active},\mathrm{accepted},\mathrm{exhausted}\}$ denote search not yet started, search in progress, search ended with a qualified checkpoint, and attempts exhausted without one, respectively. With topological order $\pi$, define
\begin{equation}
\begin{aligned}
\mathcal U_t&=\{i:\sigma_t(i)=\mathrm{active}\},\\
\mathcal R_t&=\{i:\sigma_t(i)=\mathrm{pending},\ \forall j\in d_i,\ \sigma_t(j)=\mathrm{accepted}\},\\
i_t&=\begin{cases}
\min_\pi\mathcal U_t,&\mathcal U_t\ne\varnothing,\\
\min_\pi\mathcal R_t,&\mathcal U_t=\varnothing,\ \mathcal R_t\ne\varnothing,\\
\bot,&\text{otherwise}.
\end{cases}
\end{aligned}
\label{eq:schedule}
\end{equation}
$\min_\pi$ denotes the first component in the stored order; $\bot$ denotes no available component or checkpoint. Thus an active component is processed before a newly ready one.

Let $n_i$ count started attempts, $u_i$ infrastructure attempts that did not start, and $\ell_i$ the last outcome. The next implementation action is
\begin{equation}
\operatorname{Op}(i)=
\begin{cases}
\textsc{Draft},&n_i=u_i=0,\\
\textsc{Debug},&\ell_i\in\mathcal F_{\rm impl},\\
\textsc{Check-Expectation},&\ell_i=\mathrm{contract\_failed},\\
\textsc{Improve},&\text{otherwise}.
\end{cases}
\label{eq:operator}
\end{equation}
Here \texttt{contract\_failed} denotes a valid execution whose result fails the declared test criterion. $\mathcal F_{\rm impl}$ contains execution failure, implementation failure, a missing checkpoint, and code changes that depart from the specification. \textsc{Check-Expectation} reassesses the expected effect when a valid test fails its criterion; \textsc{Debug} repairs execution or implementation errors. The agent uses the selected action and research memory to generate the next implementation. A failed design or test can prompt replanning within the revision budget, with a revised decomposition or jointly tested decisions.

Component search retains the best qualified checkpoint. Performance tests rank checkpoints by the task metric. Property checks use the specified comparison rule, with task performance breaking ties. The workflow can advance with a qualified checkpoint before the attempt \mbox{limit is reached; \cref{eq:stateupdate} gives the resulting state.}

\subsection{Full Research Loop}
Algorithm~\ref{alg:cap} combines the graph, numerical predicates, research memory, and feedback policy. Its construction step applies the planner in Section~\ref{sec:compile} followed by Algorithm~\ref{alg:construct}. Component attempts, execution launches, command durations, and outer revisions have separate limits. A failed execution counts as an attempt; an infrastructure failure before execution has a separate counter. Appendix~\ref{app:states} specifies the state updates and rules for evidence reuse.

\begin{algorithm}[htbp]
\caption{Experiments that preserve prerequisites and method revision}
\label{alg:cap}
\begin{algorithmic}[1]
\Require Task $\mathcal T$, baseline $M_0$, candidate $M$, configured resource limits
\State $(G,\pi,K,s)\gets\operatorname{Construct}(M)$ using Algorithm~\ref{alg:construct}
\State $\mathcal L\gets\varnothing$; $\widehat M\gets M_0$; $\sigma(i)\gets\mathrm{pending}$
\While{the execution and revision limits permit another step}
  \State $i\gets i_t$ from \cref{eq:schedule}
  \If{$i\ne\bot$}
    \State $\mathcal H\gets\mathcal H_t(i)$; $o\gets\operatorname{Op}(i)$ by \cref{eq:memoryprojection,eq:operator}
    \State $P_i\gets\Build(M_0,K_i)$; implement $T_i=P_i\oplus\Delta_i$ using $(o,\mathcal H)$
    \State Obtain $O_i^-,O_i^+$ and the available measurement records by \cref{eq:matchedrun}
    \State $(g_{it},q_{it})\gets$ \cref{eq:directgate,eq:enablinggate,eq:qualified}
    \State $\mathcal L(i)\gets\mathcal L(i)\mathbin{\|}F_{it}$; update $n_i,u_i$ and retained checkpoint
    \State Decide whether to advance; update $\sigma(i)$ by \cref{eq:stateupdate}
  \Else
    \If{$\forall j\in V,\ \sigma(j)=\mathrm{accepted}$}
      \State $M^+\gets\operatorname{Assemble}(\{\widehat T_j\}_{j\in V})$; evaluate $M^+$
      \State Update $\widehat M$ by \cref{eq:incumbent}; record the complete result
    \Else
      \State Report the blocking or exhausted components
    \EndIf
  \EndIf
  \If{feedback requests a permitted component or method revision}
    \State $M\gets\operatorname{Revise}(M,\mathcal L)$; reconstruct $(G,\pi,K,s)$
    \State Initialize its component histories and reuse checkpoints satisfying \cref{eq:inheritance}
  \Else
    \State Stop if component search is blocked or complete-candidate search has finished
  \EndIf
\EndWhile
\State \Return $\widehat M,\mathcal L$
\end{algorithmic}
\end{algorithm}

\FloatBarrier

\section{Plan and Experiment Records}
\label{app:schema}
For each component, DAGLoop stores its specification and the history of implementation and evaluation attempts. The following fields describe these records.

The design fields specify the component identifier, purpose, mechanism, inputs, outputs, and dependencies. The experiment fields specify the expected effect, test type, and acceptance criterion. The implementation check describes how to assess whether the code follows the design; the optimization space specifies which changes are allowed; the downstream requirements state the properties required by later components. A performance test compares task performance, and a property check checks an intermediate property. Joint tests also record their original decisions. The tuple $s_i$ groups these design and experiment fields for each component in the plan.

An attempt records the implementation action, ancestor set, checkpoints, code difference, measurements, implementation review, review of task constraints, and artifact references. Commands and execution records link observations to the corresponding attempt. The record $F_{it}$ in \cref{eq:evidence} summarizes these fields. The research memory in \cref{eq:memoryprojection} preserves detailed current-component and ancestor histories and summarizes other branches.

\section{State Transitions and Evidence Reuse}
\label{app:states}
\paragraph{Graph and semantic checks.}
The identifier map $N$ trims and lowercases strings, replaces runs of non-alphanumeric characters with a hyphen, and assigns an index-based identifier when the result is empty. The constructor computes ancestor sets, removes redundant edges, and validates acceptance prerequisites according to \cref{eq:graphclosure,eq:reduceddependencies}. Inputs and outputs are checked for structure and nonemptiness; the critic evaluates their semantic compatibility and the adequacy of the proposed tests. A merged component passes the ordinary DAG checks after replanning.

\paragraph{Control measurement.}
A root component has no prerequisites and uses the baseline as its control. For a single direct parent, the component history can provide that parent's finite metric. Components with multiple direct prerequisites require a measurement of the combined control implementation. A reported control attempt can supply this value after matching the control implementation, source state, and evaluation settings. Control reuse checks the saved measurement evidence, protocol, and restorable checkpoint. A control value that fails these checks is cleared. Protocol identity is defined by the evaluation configuration recorded for that experiment.

\paragraph{Numerical predicates and repetitions.}
The criterion for performance tests accepts one comparison of signed metric improvement or candidate metric with a finite numerical constant. The criterion for property checks accepts one scalar comparison, including a reported aggregate such as a failed-check count. Unsupported expressions return a plan-format error. Boolean, missing, and nonfinite values fail the numerical check. Equality is exact unless tolerance is expressed in the declared observable or predicate. The protocol specifies the reported statistic. For repeated measurements, the execution system checks the sample count, finite control and treatment samples, and the specified mean or median against the reported aggregate. These criteria serve as threshold filters for search decisions.

\paragraph{Qualification and retained checkpoints.}
In \cref{eq:qualified}, the task-constraint indicator excludes \texttt{violation} and \texttt{insufficient\_evidence}. Qualification also requires successful execution, an existing checkpoint, a faithful or partial implementation review, a passed component test, and no pending plan revision. The latest attempt and retained checkpoint keep separate test results.

Let $\widehat T_i^t$ be the retained qualified checkpoint after attempt $t$ and $N_i$ the component attempt limit. Let $\eta_{it}=1$ when the workflow decides to advance with a qualified checkpoint, and $0$ otherwise. After a started attempt and this decision, the state is
\begin{equation}
\sigma_{t+1}(i)=\begin{cases}
\mathrm{accepted},&\widehat T_i^t\ne\bot,\ (n_i\ge N_i\ \lor\ \eta_{it}=1),\\
\mathrm{exhausted},&n_i\ge N_i,\ \widehat T_i^t=\bot,\\
\mathrm{active},&\text{otherwise}.
\end{cases}
\label{eq:stateupdate}
\end{equation}
A non-started attempt instead increments $u_i$ and follows the infrastructure limit. On exhaustion, an existing qualified checkpoint is preserved; otherwise the workflow reports infrastructure exhaustion. For a property check, checkpoints are ranked using the specified comparison direction. Equality tests use distance from the target value, and task performance breaks ties.

\paragraph{Checkpoint reuse.}
Write $c_j$ for a component's recorded specification and acceptance requirements, $h_j$ for its retained checkpoint identity, and $\mathbf h_{d_i}$ for the ordered direct-parent checkpoint identities. An old accepted checkpoint can enter a new component history only if
\begin{equation}
\begin{aligned}
\operatorname{Inherit}(i)={}&\operatorname{ExecMatch}\ \land\ \operatorname{Restorable}(h_i)\\
&\land\;\bigl(\forall j\in K_i\cup\{i\},\ c_j^{\rm new}=c_j^{\rm old}\bigr)\\
&\land\ \bigl(\forall j\in K_i,\ \sigma^{\rm new}(j)=\mathrm{accepted},\
h_j^{\rm new}=h_j^{\rm old}\bigr)\\
&\land\ (\mathbf h_{d_i}^{\rm new}=\mathbf h_{d_i}^{\rm old}).
\end{aligned}
\label{eq:inheritance}
\end{equation}
$\operatorname{ExecMatch}$ compares execution identifiers derived from baseline records. A specification reference identifies the tested version of the design, a protocol identifier denotes the recorded evaluation configuration, and a checkpoint identifier locates the saved implementation. These checks associate reuse with the recorded experiment. Earlier experiment records remain in history when a checkpoint cannot be reused in the revised experimental context.

\section{Proofs}
\label{app:proofs}
\subsection{Baseline and Complete-Candidate Evaluations}
For the metric $f$ on feasible configurations and the marginal effect in \cref{eq:conditionalgain}, we have:
\begin{proposition}[Insufficiency of baseline and complete-candidate evaluations]
\label{prop:endpoint}
For an arbitrary deterministic metric $f$, the baseline value $f(\varnothing)$ and complete-candidate value $f(\{1,2\})$ do not determine the sign of the conditional marginal effect $\delta_{2\mid 1}$.
\end{proposition}
\begin{proof}
Consider two metrics on the three feasible configurations:
\[
\begin{array}{c|ccc|c}
 &f(\varnothing)&f(\{1\})&f(\{1,2\})&\delta_{2\mid1}\\\hline
f_A&0&0&1/2&1/2\\
f_B&0&1&1/2&-1/2
\end{array}
\]
Both metrics lie in $[0,1]$ and have identical baseline and complete-candidate values. Any rule that receives only those two values has identical inputs in the two cases, although the marginal effects have opposite signs. It therefore cannot determine the sign for all metrics.
\end{proof}

\subsection{Minimal Prerequisite Context}
A set $C\subseteq V$ is \emph{closed under prerequisites} if $j\in C$ implies $d_j^0\subseteq C$.
\begin{proposition}[Minimal prerequisite context]
\label{prop:closure}
For a DAG $G_0$, $K_i$ is the unique smallest set containing $d_i^0$ that is closed under prerequisites. Also, $i\notin K_i$, $K_i\cup\{i\}$ is closed under prerequisites, and the edge reduction in \cref{eq:reduceddependencies} preserves every ancestor set.
\end{proposition}
\begin{proof}
Every member of $K_i$ lies on a directed path to $i$. If $j\in K_i$ and $k\in d_j^0$, appending the path from $j$ to $i$ shows $k\in K_i$, so $K_i$ is closed under prerequisites. It contains $d_i^0$. Conversely, any set containing $d_i^0$ that is closed under prerequisites must contain every predecessor along each path ending at $i$, hence all of $K_i$. This establishes both uniqueness and minimality. Acyclicity gives $i\notin K_i$, and adding $i$ preserves closure because $d_i^0\subseteq K_i$.

For the reduction, process vertices in topological order. A removed parent $j\in d_i^0$ is an ancestor of another direct parent $k$. Following this relation among direct parents must terminate at a retained parent $k^*$, since the graph is finite and acyclic. By induction, the reduced graph on the earlier vertices preserves the path from $j$ to $k^*$. The retained edge $(k^*,i)$ therefore preserves reachability from $j$ to $i$. Retained edges are original edges, so the reduction creates no new reachability. Induction over the topological order $\pi$ establishes equality of all ancestor sets.
\end{proof}

\section{Definitions of Representative Tasks}
\label{app:taskdefinitions}
This appendix explains representative task objectives and evaluator interfaces for the benchmark results. The task interfaces, measurement boundaries, and reward formulas correspond to the official AutoLab repository at commit \texttt{4127da3}; the public repository supplies the task descriptions~\citep{autolabrepo}. The comparison uses the reward records in Appendix~\ref{app:fullresults} for all tasks.

\subsection{BM25 Retrieval in Go}
\label{app:bm25}
The AutoLab task uses BM25 retrieval with term frequency, inverse document frequency (IDF), and document-length normalization~\citep{autolabrepo}. For query $u$ and document $d$, the task computes the following retrieval score:
\begin{equation}
\begin{aligned}
\operatorname{BM25}(u,d)&=\sum_{w\in\operatorname{uniq}(u)}
\operatorname{IDF}(w)\,
\frac{\operatorname{tf}(w,d)(k_1+1)}
{\operatorname{tf}(w,d)+k_1(1-b+b|d|/\overline L)},\\
\operatorname{IDF}(w)&=\log\left(1+\frac{N-\operatorname{df}(w)+0.5}{\operatorname{df}(w)+0.5}\right),
\quad k_1=1.2,\quad b=0.75.
\end{aligned}
\label{eq:bm25}
\end{equation}
Here $N$ is the corpus size, $\operatorname{df}(w)$ counts documents containing $w$, $\operatorname{tf}(w,d)$ counts its occurrences in $d$, and $\overline L$ is mean document length. $\operatorname{uniq}(u)$ denotes the distinct terms in the query, so repeated query tokens are counted only once; $|d|$ is document length. The output contains at most ten positive-score documents, ordered by decreasing score and increasing document identifier on ties. An empty query returns an empty list of retrieved documents for this task.

The correctness check compares the batch query count, hit count, and checksum with the reference. Scores are rounded after multiplication by $10^6$ before checksum aggregation. The timing workload uses 4,000 documents and 40 queries, with the median over five query seeds:
\begin{equation}
m_{\rm BM25}(M)=\operatorname{median}_{r=1}^{5}
t\bigl(\operatorname{NewEngine}_M(\mathcal D),
\operatorname{RunQueries}_M(\mathcal U_r)\bigr).
\label{eq:bm25time}
\end{equation}
Timing includes index construction, query processing, ranking, and batch checksum computation. It excludes compilation, corpus and query generation, the preceding correctness check, and printing.

For measured seconds $t>0$, the official task reward is
\begin{equation}
R_{\rm BM25}(t)=
\begin{cases}
0,&2.1/t\le1,\\
\min\{1,\tfrac12\log(2.1/t)/\log(2.1/0.03)\},&2.1/t>1.
\end{cases}
\label{eq:bm25reward}
\end{equation}
The scoring anchors $2.1$ and $0.03$ are benchmark constants; the local relative improvement is computed with respect to the locally reproduced baseline instead.

\subsection{Transforms, Joins, Storage, and Model Training}
\paragraph{Fourier transform (\texttt{fft\_rust}).}
For a real signal $x\in\mathbb R^n$ with $n=32768$, the task computes the complex discrete Fourier transform
\begin{equation}
X_k=\sum_{j=0}^{n-1}x_j\exp(-2\pi\mathrm{i}jk/n),\qquad 0\le k<n.
\label{eq:fft}
\end{equation}
The metric is median transform call duration. Parseval and spectral-checksum checks use the task's $10^{-6}$ relative tolerance. The transform uses single-threaded Rust and the standard library.

\paragraph{Hash join (\texttt{hash\_join}).}
Given indexed integer key--payload tables $R=(r_a)_{a=1}^{n_R}$ and $S=(s_b)_{b=1}^{n_S}$, the task computes the matching row pairs of an inner equijoin
\begin{equation}
\mathcal J=\{(a,b):1\le a\le n_R,\ 1\le b\le n_S,\ r_a.\mathrm{key}=s_b.\mathrm{key}\}.
\label{eq:hashjoin}
\end{equation}
Every matching row pair contributes, including duplicate keys. The output is a match count and a 64-bit checksum. The benchmark uses 20,000 rows in $R$ and five million in $S$, measuring one join call after warm-up under a single-thread constraint.

\paragraph{Number-theoretic transform (\texttt{ntt\_butterfly\_cuda}).}
For $a\in\mathbb F_p^n$, the forward transform has the form
\begin{equation}
A_k=\sum_{j=0}^{n-1}a_j\omega^{jk}\pmod p,
\qquad p=2^{64}-2^{32}+1,
\label{eq:ntt}
\end{equation}
where $\omega$ is the task's forward $n$th root of unity. The CUDA implementation returns canonical field elements in natural order, in place. Correctness is bit-exact, and latency is the median CUDA-event time after warm-up, using the prescribed workspace.

\paragraph{Multi-scalar multiplication (\texttt{msm\_pippenger\_bls12\_381\_cuda}).}
The task maps BLS12-381 $G_1$ points $P_i$ and 256-bit scalars $s_i$ to
\begin{equation}
Q=\sum_{i=1}^{n}[s_i]P_i.
\label{eq:msm}
\end{equation}
The evaluator checks the affine output in the required Montgomery representation and measures median CUDA-event latency. The released starter uses naive double-and-add.

\paragraph{SSTable compaction (\texttt{sstable\_compaction\_rs}).}
Sorted string tables store ordered key--value entries on disk. This task merges several prefix-compressed tables, retains the visible latest value for each key, and removes shadowed or deleted entries. It verifies live-entry count, checksum, and output size. The median of five compaction calls includes decoding, merging, and output construction, while input generation occurs outside the measured interval.

\paragraph{Visual reasoning (\texttt{grpo\_multisource}).}
Group relative policy optimization (GRPO) fine-tunes Qwen2.5-VL-7B using Geometry3K, MathVision, and ChartQA, producing a low-rank adaptation (LoRA) adapter. The final native metric is accuracy on 100 held-out MathVista examples. Let $a$ be this accuracy and $r$ the general visual question answering (VQA) accuracy retention on 50 examples. The task reward is
\begin{equation}
R_{\rm GRPO}(a,r)=\mathbf 1[r\ge0.9]\,
\operatorname{clip}_{[0,1]}\left(\frac{a-0.20}{0.65-0.20}\right).
\label{eq:grporeward}
\end{equation}
Retention is the adapted-to-baseline VQA accuracy ratio, defined as one when baseline VQA accuracy is zero. The smaller local evaluation uses 20 MathVista examples and 10 VQA examples; the final reward is computed using the full 100-plus-50 evaluation.

\begin{landscape}
\section{Full AutoLab Results}
\label{app:fullresults}
The following table reports all 36 task scores for each of the eight configurations. For Terminus-2, each task score is the arithmetic mean of two runs in all four model/reasoning settings. DAGLoop averages two runs per task in the three reasoning-disabled settings and uses one run per task with GLM-5.2 at maximum reasoning effort. Split means average the task scores.
\begingroup
\renewcommand{\arraystretch}{1.0}
\setlength{\LTcapwidth}{\linewidth}
\setlength{\tabcolsep}{5pt}
\begin{longtable}{@{}lrrrrrrrr@{}}
\caption{All eight configurations. G = GLM-5.2; D4 = DeepSeek-V4-Flash-0731; D4.1 = DeepSeek-V4.1-Flash. G+ uses maximum reasoning effort; G0, D4, and D4.1 have reasoning disabled.}\\
\toprule
 & \multicolumn{4}{c}{\textbf{DAGLoop}} & \multicolumn{4}{c}{Terminus-2} \\
\cmidrule(lr){2-5}\cmidrule(lr){6-9}
Task & G+ & G0 & D4 & D4.1 & G+ & G0 & D4 & D4.1 \\
\midrule
\endfirsthead
\toprule
 & \multicolumn{4}{c}{\textbf{DAGLoop}} & \multicolumn{4}{c}{Terminus-2} \\
\cmidrule(lr){2-5}\cmidrule(lr){6-9}
Task & G+ & G0 & D4 & D4.1 & G+ & G0 & D4 & D4.1 \\
\midrule
\endhead
\texttt{adaptive\_compression} & 0.5504 & 0.4538 & 1.0000 & 0.5659 & 0.4968 & 0.3957 & 0.2887 & 0.4982 \\
\texttt{adversarial\_splay} & 0.5648 & 0.5000 & 0.5317 & 0.5739 & 0.5841 & 0.5235 & 0.5689 & 0.4861 \\
\texttt{aes128\_ctr} & 0.6103 & 0.6056 & 0.5205 & 0.6177 & 0.5794 & 0.2521 & 0.4454 & 0.5149 \\
\texttt{agent\_tool\_routing} & 0.4637 & 0.2646 & 0.2026 & 0.3734 & 0.2403 & 0.2245 & 0.1989 & 0.3471 \\
\texttt{bm25\_search\_go} & 1.0000 & 0.6853 & 0.2725 & 0.4763 & 1.0000 & 0.3736 & 0.5377 & 0.4050 \\
\texttt{bvh\_raytracer} & 0.3285 & 0.2906 & 0.3233 & 0.3218 & 0.4551 & 0.2834 & 0.2639 & 0.2558 \\
\texttt{concurrent\_kv\_wal} & 0.8134 & 0.4009 & 0.3278 & 0.2614 & 0.6040 & 0.6384 & 0.4751 & 0.5043 \\
\texttt{data\_select\_ifeval} & 0.6177 & 0.3405 & 0.2417 & 0.3774 & 0.8581 & 0.0316 & 0.2758 & 0.2758 \\
\texttt{discover\_sorting} & 0.8500 & 1.0000 & 0.9250 & 1.0000 & 0.8500 & 0.8500 & 0.8500 & 0.9250 \\
\texttt{fft\_rust} & 0.5225 & 0.4762 & 0.5037 & 0.5052 & 0.5266 & 0.2346 & 0.2438 & 0.4857 \\
\texttt{flash\_attention} & 0.4467 & 0.5775 & 0.3260 & 0.3643 & 0.4695 & 0.2156 & 0.3467 & 0.2355 \\
\texttt{flux2\_klein\_lora} & 1.0000 & 0.4477 & 0.7505 & 1.0000 & 0.0000 & 0.2475 & 0.0000 & 0.3403 \\
\texttt{fredkin\_sort\_network} & 1.0000 & 0.5867 & 0.8000 & 0.8333 & 0.8667 & 0.4333 & 0.5334 & 0.5667 \\
\texttt{gaussian\_blur} & 0.5576 & 0.4859 & 0.3672 & 0.5507 & 0.5574 & 0.4701 & 0.4732 & 0.4068 \\
\texttt{grpo\_multisource} & 0.9111 & 0.8667 & 0.8722 & 0.8556 & 0.8778 & 0.8223 & 0.8445 & 0.8111 \\
\texttt{hash\_join} & 0.5595 & 0.5031 & 0.5425 & 0.5045 & 0.6402 & 0.5274 & 0.5671 & 0.6267 \\
\texttt{huffman\_canonical\_decode\_cuda} & 0.3850 & 0.3489 & 0.3344 & 0.1681 & 0.3798 & 0.2979 & 0.2642 & 0.2873 \\
\texttt{icp\_correspondence\_step\_cuda} & 0.5063 & 0.5048 & 0.5133 & 0.5040 & 0.4994 & 0.4700 & 0.2926 & 0.4989 \\
\texttt{levenshtein\_distance} & 0.3861 & 0.3820 & 0.3921 & 0.4170 & 0.3874 & 0.2082 & 0.1901 & 0.3351 \\
\texttt{llm\_online\_serving} & 0.7378 & 0.5362 & 0.5296 & 0.8941 & 0.0006 & 0.0005 & 0.0001 & 0.0000 \\
\texttt{moving\_mnist\_world\_model} & 0.2280 & 0.1772 & 0.2385 & 0.1351 & 0.2579 & 0.3209 & 0.3189 & 0.2105 \\
\texttt{msm\_pippenger\_bls12\_381\_cuda} & 0.3991 & 0.4404 & 0.6596 & 0.4443 & 0.3824 & 0.1896 & 0.2647 & 0.2205 \\
\texttt{multilingual\_ocr} & 1.0000 & 1.0000 & 0.8192 & 0.9690 & 1.0000 & 0.6362 & 0.5563 & 0.9949 \\
\texttt{ntt\_butterfly\_cuda} & 0.5353 & 0.4517 & 0.4757 & 0.5135 & 0.5125 & 0.4784 & 0.4964 & 0.5579 \\
\texttt{radix\_sort} & 0.4334 & 0.3626 & 0.3526 & 0.3974 & 0.4108 & 0.3243 & 0.2857 & 0.2758 \\
\texttt{regex\_engine} & 0.9614 & 0.7156 & 0.6316 & 0.7575 & 0.8803 & 0.5153 & 1.0000 & 0.8454 \\
\texttt{resnet\_bit\_flip} & 0.9750 & 0.6875 & 0.3938 & 0.9188 & 0.8813 & 0.6563 & 0.9188 & 0.5313 \\
\texttt{safety\_router} & 0.9921 & 0.8828 & 0.8749 & 0.9335 & 0.9921 & 0.9921 & 0.3750 & 0.9882 \\
\texttt{scaling\_law} & 0.5844 & 0.2481 & 0.6536 & 0.8085 & 0.6054 & 0.6084 & 0.5828 & 0.3967 \\
\texttt{sha256\_throughput} & 0.2769 & 0.2767 & 0.2563 & 0.2896 & 0.1799 & 0.0117 & 0.0208 & 0.2780 \\
\texttt{smallest\_game\_player} & 0.5786 & 0.7639 & 0.0608 & 0.4108 & 0.4740 & 0.0000 & 0.0000 & 0.2830 \\
\texttt{sstable\_compaction\_rs} & 0.1780 & 0.0665 & 0.1100 & 0.1566 & 0.3479 & 0.2146 & 0.0775 & 0.1792 \\
\texttt{stack\_machine\_golf} & 1.0000 & 1.0000 & 1.0000 & 0.7741 & 1.0000 & 1.0000 & 0.5324 & 1.0000 \\
\texttt{toy\_isa\_opt} & 1.0000 & 0.9996 & 0.9910 & 0.9993 & 1.0000 & 0.9913 & 0.9797 & 0.9953 \\
\texttt{vliw\_scheduler} & 1.0000 & 0.9868 & 0.9723 & 1.0000 & 1.0000 & 0.7306 & 0.9806 & 0.9906 \\
\texttt{z\_order\_range\_scan} & 0.2982 & 0.2331 & 0.1646 & 0.2747 & 0.1953 & 0.1887 & 0.0774 & 0.1814 \\
\midrule
Mean (36) & 0.6459 & 0.5430 & 0.5258 & 0.5819 & 0.5831 & 0.4266 & 0.4202 & 0.4926 \\
\bottomrule
\end{longtable}

\endgroup
\end{landscape}

\clearpage
\section{Additional FML-Bench Results}
\label{app:fmlresults}
We additionally ran DAGLoop on 18 FML-Bench tasks with the task and repository specified~\citep{zou2026fmlbench}. The results cover 13 GPU tasks and five CPU tasks. Table~\ref{tab:fmlresults} reports their native metrics separately from AutoLab. We report one native result per task with the baseline from the same run.

\begin{table}[H]
\centering
\caption{FML-Bench native results with corresponding local baselines.}
\label{tab:fmlresults}
\normalsize
\begingroup
\setlength{\tabcolsep}{3.1pt}
\renewcommand{\arraystretch}{1.18}
\begin{tabular*}{\linewidth}{@{\extracolsep{\fill}}llrr@{}}
\toprule
Task & Metric & Baseline & \textbf{DAGLoop}\\
\midrule
\multicolumn{4}{@{}l}{\textbf{GPU tasks (13)}}\\[2pt]
Opacus & Acc. $\uparrow$ & 0.5875 & \textbf{0.6024}\\
DomainBed-CM & In-acc. $\uparrow$ & 0.2613 & \textbf{0.7010}\\
ART & Defense $\uparrow$ & 0.3131 & \textbf{0.8719}\\
OpenOOD & AUROC $\uparrow$ & 0.8905 & \textbf{0.9209}\\
PyCIL & Inc.-acc. $\uparrow$ & 0.5957 & \textbf{0.6016}\\
USB & Acc. $\uparrow$ & 0.0783 & \textbf{0.0961}\\
EasyFSL & Acc. $\uparrow$ & 0.6534 & \textbf{0.7034}\\
PFLlib & Acc. $\uparrow$ & 0.4530 & \textbf{0.5211}\\
DomainBed-OH & Avg.-acc. $\uparrow$ & 0.5962 & \textbf{0.6280}\\
PrivacyMeter & AUC gap $\downarrow$ & 0.3177 & \textbf{0.0252}\\
Lightly & Acc. $\uparrow$ & 0.7560 & \textbf{0.7886}\\
Solo-learn & Linear acc. $\uparrow$ & 0.4667 & \textbf{0.4879}\\
OpenUnlearning & KS $p$ $\uparrow$ & $\approx 0$ & \textbf{0.3890}\\
\midrule
\multicolumn{4}{@{}l}{\textbf{CPU tasks (5)}}\\[2pt]
Fairlearn & $|\mathrm{DP}|$ $\downarrow$ & 0.1866 & \textbf{0.0029}\\
gCastle & SHD $\downarrow$ & 72 & \textbf{24}\\
AIF360 & $|\mathrm{AOD}|$ $\downarrow$ & 0.2399 & \textbf{0.0209}\\
CausalML & MAE $\downarrow$ & 1.2963 & \textbf{1.2680}\\
Cont.-Learn. & Avg.-acc. $\uparrow$ & 0.2841 & \textbf{0.5935}\\
\bottomrule
\end{tabular*}
\endgroup
\vspace{3pt}
\begin{minipage}{\linewidth}
CM and OH denote ColoredMNIST and OfficeHome, respectively. \textbf{Metrics:} DP = demographic-parity difference; AOD = average odds difference; SHD = structural Hamming distance. Acc. = accuracy; In-acc. = mean accuracy on the splits designated ``in'' by DomainBed; Inc.-acc. = average incremental accuracy; Linear acc. = linear evaluation accuracy; Avg.-acc. = average accuracy. AUROC = area under the receiver operating characteristic curve; AUC gap = the task-defined area-under-curve gap; MAE = mean absolute error; KS $p$ = Kolmogorov--Smirnov test $p$-value. Defense denotes the defense score defined by the corresponding task.
\end{minipage}
\end{table}

\end{document}